\documentclass[twoside]{article}
\usepackage{aistats2016}

\def\bbbr{{\rm I\!R}}

\newcommand{\GP}{\mathcal{GP}}

\newcommand{\bk}{\textbf{k}}
\newcommand{\bw}{\textbf{w}}
\newcommand{\bK}{\textbf{K}}
\newcommand{\latentVector}{\textbf{x}}
\newcommand{\bX}{\textbf{X}}
\newcommand{\bx}{\textbf{x}}
\newcommand{\by}{\textbf{y}}

\newcommand{\I}{\mathcal{I}}

\global\long\def\inputSpace{\mathcal{X}}

\global\long\def\dataScalar{y}

\global\long\def\latentScalar{x}

\global\long\def\latentVector{\mathbf{\latentScalar}}

\global\long\def\latentForce{f}

\usepackage{color}
\usepackage{verbatim}

\definecolor{brown}{rgb}{0.9,0.59,0.078}
\definecolor{ironsulf}{rgb}{0,0.7,.5}
\definecolor{lightpurple}{rgb}{0.156,0,0.245}

\ifdefined\blackBackground
\definecolor{colorOne}{rgb}{0, 1, 1}
\definecolor{colorTwo}{rgb}{1, 0, 1}
\definecolor{colorThree}{rgb}{1, 1, 0}
\definecolor{colorTwoThree}{rgb}{1, 0, 0}
\definecolor{colorOneThree}{rgb}{0, 1, 0}
\definecolor{colorOneTwo}{rgb}{0, 0, 1}
\else
\definecolor{colorOne}{rgb}{1, 0, 0}
\definecolor{colorTwo}{rgb}{0, 1, 0}
\definecolor{colorThree}{rgb}{0, 0, 1}
\definecolor{colorTwoThree}{rgb}{0, 1, 1}
\definecolor{colorOneThree}{rgb}{1, 0, 1}
\definecolor{colorOneTwo}{rgb}{1, 1, 0}
\fi

\ifdefined\blackBackground

\else

\fi

\global\long\def\cut#1{}

\global\long\def\detail#1{}

\newtheorem{definition}{Definition}
\newtheorem{proof}{Proof}
\newtheorem{proposition}{Proposition}

\newcommand{\erfc}{erfc}

\usepackage{graphicx} 
\usepackage{subfigure} 

\usepackage{natbib}

\usepackage{algorithm}
\usepackage{algorithmic}
\usepackage{multirow}

\begin{document}

%

%
\runningauthor{Surname 1, Surname 2, Surname 3, ...., Surname n}

\twocolumn[

\aistatstitle{Batch Bayesian Optimization via Local Penalization}

\aistatsauthor{ Javier Gonz\'alez \And Zhenwen Dai \And Philipp Hennig \And Neil Lawrence}

\aistatsaddress{ University of Sheffield \And University of Sheffield \And Max Planck Institute\\ for Intelligent Systems \And University of Sheffield } ]

\begin{abstract}
The popularity of Bayesian optimization methods for efficient
exploration of parameter spaces has lead to a series of papers
applying Gaussian processes as surrogates in the optimization of
functions. However, most proposed approaches only allow the
exploration of the parameter space to occur
\emph{sequentially}. Often, it is desirable to simultaneously propose
\emph{batches} of parameter values to explore. This is particularly
the case when large parallel processing facilities are
available. These could either be computational or physical facets of
the process being optimized. Batch methods, however, require the modeling
of the interaction between the different evaluations in the batch,
which can be expensive in complex scenarios. We investigate this issue
and propose a highly effective heuristic based on an estimate of the
function's \emph{Lipschitz constant} that captures the most important
aspect of this interaction---local repulsion---at negligible
computational overhead. A \emph{penalized acquisition function} is
used to collect batches of points minimizing the non-parallelizable
computational effort. The resulting algorithm compares very well, in
run-time, with much more elaborate alternatives.
\end{abstract}

\section{Introduction}\label{sec:introduction}

Many problems, such as the configuration of machine learning
algorithms \citep{snoek68} or the experimental design of biological
experiments \citep{gonzalez2014} require the optimization of an
unknown, possibly noisy, function $f$. Bayesian optimization (BO) has
emerged in this scenario as an efficient heuristic to optimize $f$ if
function evaluations are costly and the overall number of evaluations
must be kept low \citep{Jones:1998:EGO:596070.596218}.

The task is to solve the global optimization problem of finding
\begin{equation}\label{eq:problem}
\latentVector_{M} = \arg \max_{\latentVector \in {\inputSpace}} \latentForce(\latentVector).
\end{equation}
We assume that $f$ is a \emph{black-box} from which only perturbed
evaluations of the type $\dataScalar_i = \latentForce(\latentVector_i) + \epsilon_i$, with
$\epsilon_i \sim\mathcal{N}(0,\sigma^2)$, are available. We will
assume that the objective of interest can be described well by a
L-Lipschitz continuous function $f: {\inputSpace} \to \bbbr$ defined on
a compact subset ${\inputSpace} \subseteq \bbbr^d$.

In sequential BO the goal is to make a series of evaluations
$\latentVector_1,\dots,\latentVector_N$ of $f$ such that the maximum of $f$ is
evaluated as quickly as possible. After $n$ points are available, BO
proposes a new location $\latentVector_{n+1}$ using a probabilistic model for
$f$, conditioned on all previous observations $\mathcal{D}_{n} =
\{(\latentVector_i,\dataScalar_i)\} _{i=1}^{n}$. Typically the model is a Gaussian process
(GP) $p(f) = \mathcal{GP}(\mu; k)$ with mean function $\mu$ and
positive-definite covariance function (kernel) $k$ that in this
work we assume is stationary. Under Gaussian likelihoods, the
posterior distribution of $f$ (for a sample of size $n$) is also a GP, with
posterior mean and variance given by
$$\mu_n(\latentVector^{*}) = \bk_{n}(\latentVector^{*})^\top[\bK_{n} + \sigma_n^2 \textbf{I}]^{-1}\by_n$$
and
$$\sigma_n^2(\latentVector^{*})=k(\latentVector^{*},\latentVector^{*})-\bk_{n}(\latentVector^{*})^\top[\bK_{n}+\sigma_n^2 \textbf{I}]^{-1}\bk_{n}(\latentVector^{*}),$$
where $\bK_{n}$ is the matrix such that $(\bK_{n})_{ij}=k(\latentVector_i,\latentVector_j)$,  $\bk_{n}(\latentVector^{*}) = [k(\latentVector_1,\latentVector^{*}),\dots,k(\latentVector_n,\latentVector^{*})]^\top$ \citep{Rasmussen:2005:GPM:1162254} and $\latentVector^{*}$ is the point where the GP is evaluated.  

This posterior is used to form the acquisition function $\alpha(\latentVector;
\I_n)$, where $\I_n$ represents the available data set
$\mathcal{D}_{n}$ and the GP structure (kernel, likelihood and
parameter values) when $n$ data points are available. The next
evaluation is placed at the (numerically estimated) global maximum
$\latentVector_{n+1}$ of this acquisition function. A number of possible
acquisition functions are now available, ranging from fast heuristics
\citep{osborne_bayesian_2010,Jones:1998:EGO:596070.596218} to non-local
entropy-based approaches \citep{HennigSchuler2012,NIPS2014_5324}.

While the goal of Bayesian optimization is to keep the number of
evaluations of $f$ as low as possible, in high-dimensional and or
otherwise complex problems, the number of required evaluations can
still be considerable. Parallel approaches arise as the natural
solution to circumvent the computational bottleneck around these
evaluations of $f$. We focus on cases in which the cost of evaluating
$f$ in a batch of points of size $n_b$ is the same as evaluating $f$
in a single point. Such scenarios appear, for instance, in the
optimization of computer models where several cores are available to
run in parallel, or in wet-lab experiments when the cost of testing
one experimental design is the same as testing a batch of them. In
these settings, the set of available pairs $\{(\latentVector_i,\dataScalar_i)\}_{i=1}^n$
can be augmented with the evaluations of $f$ on batches of data points
$\mathcal{B}_t^{n_b} =\{\latentVector_{t,1},\dots,\latentVector_{t,nb} \}$, for
$t=1,\dots,m$, rather than on single observations. Our goal here is to
define a design rule for such batches
$\mathcal{B}_1^{n_b},\dots,\mathcal{B}_m^{n_b}$. The batch selection
problem can be generalized further, \emph{e.g.}~by adapting the batch
size \citep{DBLP:conf/icml/AzimiJF12} or by collecting batches
asynchronously \citep{ginsbourger:hal-00507632,
  conf/lion/JanusevskisRGG12,snoek68}. For simplicity of exposition
these ideas will not feature further here.

\subsection{Optimal batch design and previous work}

The goal of any batch criterion is to mimic the decisions that would be made under the equivalent (optimal) sequential policy: Consider the choice of selecting $\latentVector_{t,k}$, the $k$-th element of the $t$-th batch. Under a sequential policy, in which the evaluations of $f$ at all locations prior to $\latentVector_{t,k}$ are available, the decision is to take $\latentVector_{t,k}$ as the maximizer of $\alpha(\latentVector;\I_{t,k-1})$. In the batch case, the decision about where to collect $\latentVector_{t,k}$ has to incorporate the uncertainty about the locations  $\latentVector_{t,1},\dots,\latentVector_{t,k-1}$, and the outcomes of the evaluation of $f$ there. Iteratively marginalizing these sources of uncertainty gives
\begin{eqnarray}\label{eq:optimal_batch}
\latentVector_{t,k} &= &\arg \max_{\latentVector \in {\inputSpace}} \int \alpha(\latentVector;\I_{t,k-1})\\ \nonumber
 & &\prod_{j=1}^{k-1}p(y_{t,j}|\latentVector_{t,j},\I_{t,j-1}) p(\latentVector_{t,j}|\I_{t,j-1} )\mbox{d}\latentVector_{t,j} \mbox{d}y_{t,j},\nonumber
\end{eqnarray}
where 
$$p(y_{t,j}|\latentVector_{t,j},\I_{t,j-1})= \mathcal{N} \left(y_{t,j};\mu_n(\latentVector_{t,j}),\sigma_n^2(\latentVector_{t,j}) \right)$$ 
is the predictive distribution of the GP at $\latentVector_{t,j}$ when a total of $n$ points are available and 
$$p(\latentVector_{t,j}|\I_{t,j-1}) = \delta (\latentVector_{t,j} - \arg \max_{\latentVector \in {\inputSpace}} \alpha(\latentVector;\I_{t,j-1}) )$$ reflects the optimization step required to obtain $\latentVector_{t,j}$ after the evaluations of $f$ at previous batch-elements have been marginalized. 

The optimization in Eq.~(\ref{eq:optimal_batch}) is intractable even for small batch-sizes, due to the optimization-marginalization loop required to obtain $\latentVector_{t,k}$. The literature in batch BO has  tried to avoid this computational burden by means of different strategies, most of which involve the explicit use of the predictive distributions $p(y_{t,j}|\latentVector_{t,j},\I_{t,j-1})$, for $j=1,\dots,n_b$. Exploratory approaches \citep{schonlau,journals/corr/abs-1304-5350} search for a reduction in system uncertainty. This is using the property that the variance of $p(y_{t,j}|\latentVector_{t,j},\I_{t,j-1})$ does not depend on the value of the objective there. Other methods use $p(y_{t,j}|\latentVector_{t,j},\I_{t,j-1})$ to generate `fake' observations of the model \citep{DBLP:conf/icml/AzimiJF12,DBLP:journals/corr/abs-1110-3347,Bergstra+al-NIPS2011} and avoid the marginalization step. In statistics, the suitability of the expected improvement utility has been studied for the design of batches \citep{ChevalierG13,Frazier2012}. In contrast to the previous mentioned works, these methods use the joint distribution of $y_{t_1},\dots y_{t,nb}$ to simultaneously optimize elements on the batch \citep{DBLP:conf/nips/AzimiFF10}. These non-greedy strategies are very well founded from a theoretical perspective in practice but tend to scale poorly with the dimension of the problem and the sizes of the batches. Other theoretical properties of batch BO have been studied in the context of Bayesian networks \citep{firstPBO}, multi-armed bandits \citep{DBLP:conf/icml/DesautelsKB12}, and the optimal balance between exploration and exploitation in batch designs \citep{DBLP:conf/pkdd/JalaliAFZ13}. 

\subsection{Goal and Contributions of this work}
Using $p(y_{t,j}|\latentVector_{t,j},\I_{t,j-1})$ to model the interaction between batch elements has a computational overhead of $\mathcal{O}(n^3)$, since the GP needs to be updated after each batch location is selected to jointly optimize all the elements in the batch. The motivation of this work is to develop a \emph{heuristic approximation} of Eq.~(\ref{eq:optimal_batch}) at lower computational cost, while incorporating information about global properties of $f$ from the GP model into the batch design.

Our approach rests on the hypothesis that $f$ is a Lipschitz
continuous function, which is a common assumption in global
optimization \citep{reference/opt/2009}. For easy reference: a
real-valued function $f: {\inputSpace} \to \bbbr$ on a compact subset
${\inputSpace}\subseteq \bbbr^d$ of the $d$-dimensional real vector
space is said to be $L$-Lipschitz if it satisfies
\begin{equation}\label{eq:lipschitz}
|\latentForce(\latentVector_1) - \latentForce(\latentVector_2) | \leq L \|\latentVector_1 -\latentVector_2 \|,\qquad\forall\; \latentVector_1,\latentVector_2\in\mathcal{X}
\end{equation}
where $L$ is a global positive constant, and $\| \cdot\|$ is the $\ell^2$-norm on $\bbbr^d$, a property that has been previously exploited in global optimization \citep{opac-b1088635,opac-b1107684}.

In the context of parallelizing Bayesian optimization, a beneficial aspect of the Lipschitzian assumption is that it naturally allows us to place bounds on how far the optimum of $f$ is from a certain location. See Figure~\ref{figure:exclusion_zones} for details. As explained below, this information can be used to define policies to collect a batch of points multiple steps ahead without evaluating $f$, by mimicking the hypothesized behavior of a sequential policy. The main challenge is that, in practice, the constant $L$ is unknown. In the
literature, this problem has been addressed from different angles \citep{reference/opt/2009}. We explore a new alternative: inferring the Lipschitz constant directly from the Gaussian process model for $f$.

\begin{figure}[t!]
	\centering
  \includegraphics[width=.45\textwidth]{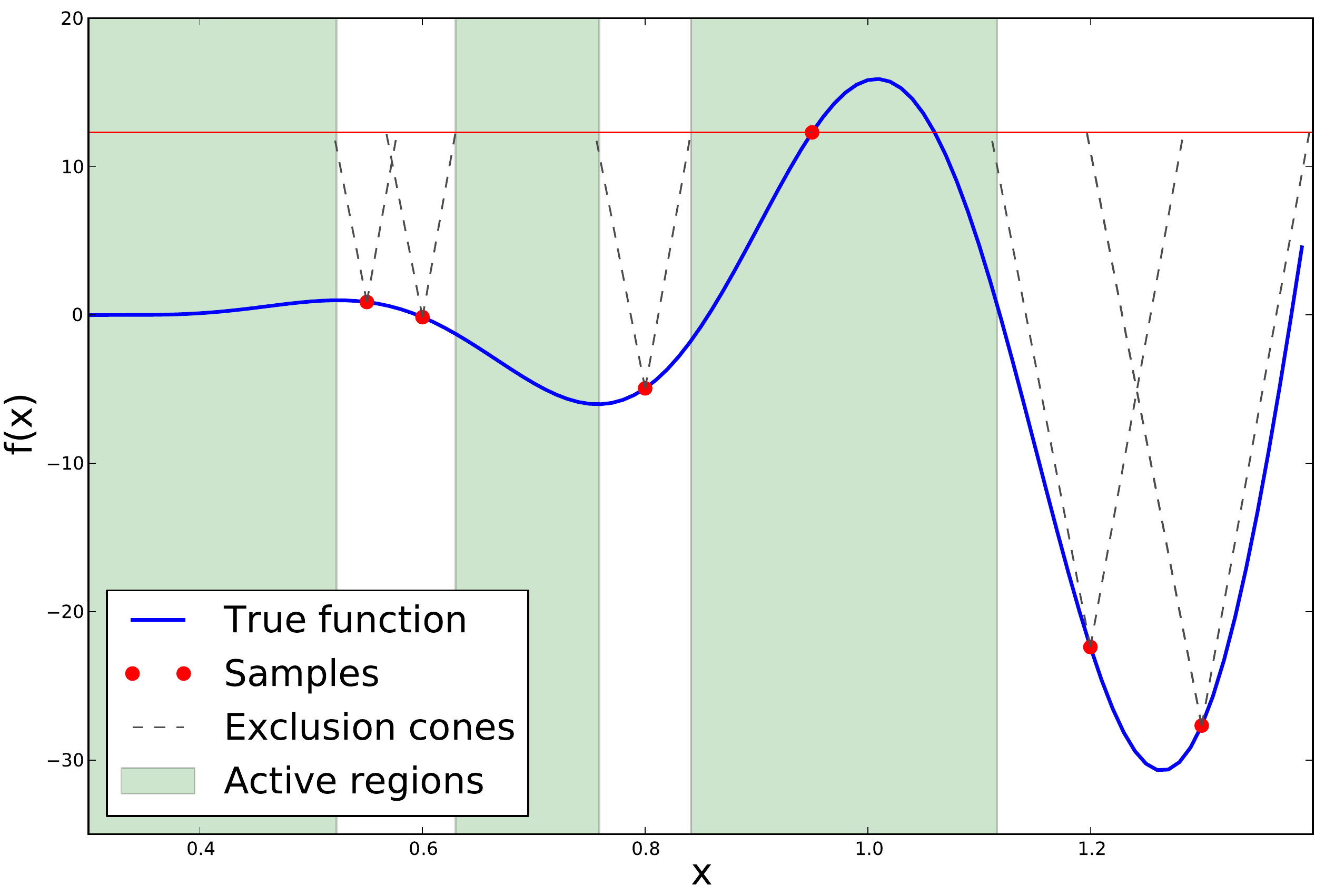}
\caption{  Forrester function $f(x) = (6x-2)^2 \sin(12x-4)$ in the interval $[0.3, 01,4]$. We take 6 evaluations $\bx_1,\dots, \bx_6$ of the function, $M=\max_i f(\bx_i)$ and $L=400$. The exclusion zones for the maximum of $f$ determined by the balls $B_r({\bx_i})$ are shown.\vspace{-2mm}}
	\label{figure:exclusion_zones}
\end{figure}

Our contributions are: (i) A new batch BO heuristic, BBO-LP, that
selects batches of points by an iterative
\emph{maximization-penalization} loop around the the acquisition
function. This leads to efficient parallelization of BO and can be
used with any acquisition function. (ii) A probabilistic framework to
approximately infer the Lipschitz constant of $f$, termed GP-LCA, that
uses the properties of the gradients of the GP. The inferred value of
$L$ is used to improve batch selection.  (iii) A python implementation
of several batch BO methods is published in
conjunction with this work.\footnote{http://sheffieldml.github.io/GPyOpt/}
(iv) Confirmation of the effectiveness of the approach is demonstrated through several
simulated experiments, an algorithm configuration problem, and a real
wet-lab experimental design. In particular, the local penalization
approach performs equal or better than current batch BO methods in
terms of the convergence to the maximum, but shows better performance
in terms of gained \emph{information per second}.

\begin{figure*}[t!]
  \centering
  \includegraphics[width=.95\textwidth]{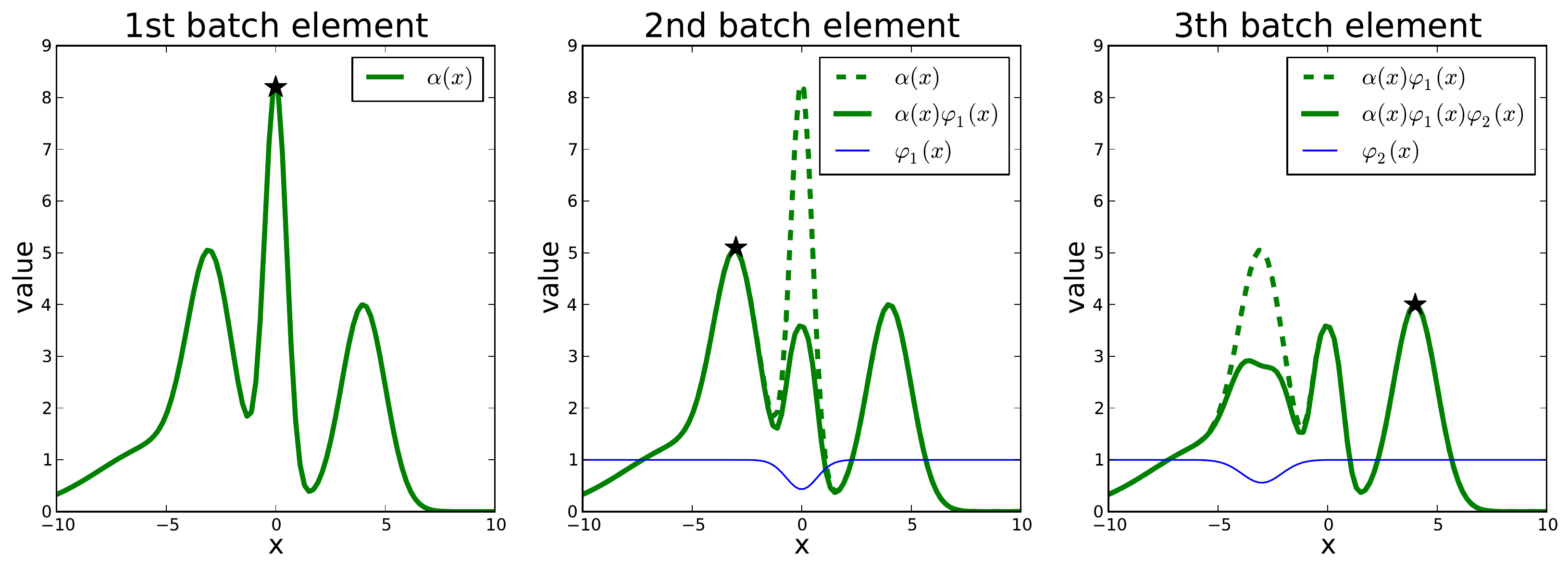}
  \caption{Illustration of three iterations of the \emph{maximization-penalization} loop. The main task of good batch design is to explore the modes of the acquisition function, achieved by iterative maximization (black stars) and penalization (using $\varphi_1(\latentVector), \varphi_2(\latentVector)$) of the acquisition function $\alpha(\latentVector)$.}
  \label{figure:batch}
\end{figure*}

\section{Maximization-Penalization Strategy for Batch Design}


The intuition behind our approach is that for most GP priors in
practical use for BO, the dominant effect of a function evaluation on
the acquisition function is a \emph{local exclusion} around the new
evaluation. This shape of the acquisition function will be modeled
through the Lipschitz properties of $f$, to distribute the elements in
each batch. This should be understood as a heuristic to the shape of
$\alpha(\latentVector;\I_{t,k-1})$ if all previous observations were available,
mimicking the effect a sequential policy. This is especially useful in
cases in which the acquisition function shows \emph{multi-modal
  shape}, a common situation in the first iterations of BO algorithms.
The following definition is helpful for the formalization of the
algorithm:
\begin{definition}
A function $\varphi(\latentVector;\latentVector_j)$, $\latentVector\in \inputSpace$, is a \emph{local penalizer} of a generic acquisition function $\alpha(\latentVector)$ at $\latentVector_j$ if $\varphi(\latentVector;\latentVector_j)$ is differentiable, $0\leq\varphi(\latentVector;\latentVector_j)\leq 1$ and $\varphi(\latentVector;\latentVector_j)$ is an non-decreasing  function in $\|\latentVector-\latentVector_j \|$.
\end{definition}

We propose to replace the \emph{maximization-marginalization} loop in Eq.~\ref{eq:optimal_batch} by a \emph{maximization-penalization} strategy: while the optimization is carried out in a similar fashion, the marginalization step is replaced by the direct penalization of $\alpha(\latentVector;\I_{t,k-1})$ around its most recent maximum, \emph{i.e}, the previous batch element. Figure \ref{figure:batch} gives a graphical illustration. The maximization-penalization strategy selects $\latentVector_{t,k}$ as 
\begin{equation}\label{eq:penalized_acquisition}
\latentVector_{t,k} =\arg \max_{x \in \inputSpace} \left\{g(\alpha(\latentVector; \mathcal{I}_{t,0}))\prod_{j=1}^{k-1}\varphi(\latentVector;\latentVector_{t,j})\right\},
\end{equation}
where $\varphi(\latentVector;\latentVector_{t,j})$ are local local penalizers centered at $\latentVector_{t,j}$ and $g:\bbbr \rightarrow \bbbr^+$ is a differentiable transformation of $\alpha(\latentVector)$ that keeps it strictly positive without changing the location of its extrema. We will use $g(z) = z$ if $\alpha(\latentVector)$ is already positive and the \emph{soft-plus} transformation $g(z)= \ln(1+e^z)$ elsewhere. This does not require re-estimation of the GP model after each location is selected, just a new optimization of the penalized utility. 

The effect of a local penalizer is to smoothly reduce the value of the acquisition function in a neighborhood of $\latentVector_j$. A `good' local penalizer centered at $\latentVector_j$ should reflect the belief about the distance from $\latentVector_j$ to $\latentVector_M$: If we suspect that $\latentVector_M$ is far from $\latentVector_j$, a broad $\varphi(\latentVector;\latentVector_j)$ will discard a large portion of $\inputSpace$ in which we don't need to collect any sample.  On the other hand, if we believe that $\latentVector_M$ and $\latentVector_j$ are close, ideally we want to minimize the penalization of $\alpha(\latentVector)$ and keep collecting samples is a close neighborhood. This local penalization mimics the acquisition function's dynamics under a sequential policy in the following sense: the modes of the acquisition functions correspond to regions in which either $\mu_n(\latentVector)$ or $\sigma^2_n(\latentVector)$ (or both) are large. Evaluating, for instance, where $\sigma_n(\latentVector)$ is large will reduce uncertainty in that region, decreasing $\alpha(\latentVector)$ in a neighborhood. The functions $\varphi(\latentVector;\latentVector_j)$ are surrogates for this neighborhood. 

\subsection{Choosing Local Penalizers $\varphi(\latentVector;\latentVector_j)$}
We now construct penalizing functions $\varphi(\latentVector;\latentVector_j)$ that incorporate into $\alpha(\latentVector)$ the current belief about the distance from the batch locations to $\latentVector_M$. Take $M = \max_{\latentVector \in {\inputSpace}} \latentForce(\latentVector)$, and a valid Lipschitz constant $L$.  Consider the ball
\begin{equation}\label{eq:radious}
B_{r_j}(\latentVector_j)= \{\latentVector\in \inputSpace : \|\latentVector_j-\latentVector \| \leq r_j \}
\end{equation}
where
$$r_j = \frac{M-\latentForce(\latentVector_j)}{L}.$$
To simplify the notation we write $r_j=r(\latentVector_j)$ for the radius of the ball around $\latentVector_j$. If $f$ in (\ref{eq:radious}) is the true optimization objective, then $\latentVector_M \notin B_{r_j}(\latentVector)$---otherwise the Lipschitz condition would be violated. The size of $B_{r_j}(\latentVector_j)$ depends on $L$, $M$ and the value of $f$ at $\latentVector_j$. Both large variability in $f$ (large $L$) and proximity of $\latentForce(\latentVector_j)$ to the optimum $M$ shrink $B_{r_j}(\latentVector_j)$.

In the BO context, under the assumption $\latentForce(\latentVector) \sim \GP(\mu(\latentVector),k(\latentVector,\latentVector'))$, we choose $\varphi(\latentVector;\latentVector_j)$ as the probability that $\latentVector$, any point in $\inputSpace$ that is a potential candidate to be a maximum, does not belong to $B_{r_j}(\latentVector_j)$:
\begin{equation}\label{eq:penalizer}
\varphi(\latentVector;\latentVector_j)  = 1- p (\latentVector  \in B_{r_j}(\latentVector_j) ).
\end{equation}

The following proposition (proof in Supp. Materials~\ref{appen:proof}) shows that this local penalizer can be computed in closed form.

\begin{proposition}
Let $\latentForce(\latentVector)$ be a $\GP$ with posterior mean $\mu_n(\latentVector)$ and posterior variance $\sigma^2_n(\latentVector)$. The function $\varphi(\latentVector;\latentVector_j)$ in Eq. (\ref{eq:penalizer}) is a valid local penalizer of $\alpha(\latentVector)$ at $\latentVector_j$ such that:
$$\varphi(\latentVector;\latentVector_j) = \frac{1}{2}\erfc\left(-z \right)$$
where $z = \frac1{\sqrt{2\sigma_n^2(\latentVector_j)}} \left(L\|\latentVector_j-\latentVector \| - M + \mu_n(\latentVector_j) \right),$
 for erfc the complementary error function, $M = \max_{\latentVector \in {\inputSpace}} \latentForce(\latentVector)$ and $L$ a valid Lipschitz constant.
\end{proposition}

The functions $\varphi(\latentVector;\latentVector_j)$ thus create exclusion zones whose size is governed by $L$. If $\mu_n(\latentVector_j)$ is close to $M$, then $\varphi(\latentVector;\latentVector_j)$ will have a smaller and more localized effect on $\alpha(\latentVector)$ (a smaller exclusion area). On the other hand, if $\mu_n(\latentVector_j)$ is far from $M$, $\varphi(\latentVector;\latentVector_j)$  will produce a wider yet less intense correction on $\alpha(\latentVector)$. The value of $L$ also affects the size of the effect of $\varphi(\latentVector;\latentVector_j)$ on $\alpha(\latentVector)$, decreasing it as $L$ increases.

\subsection{Selecting the parameters $L$ and $M$}\label{sec:gp_lca}
The values of $M$ and $L$ are unknown in general. To approximate $M$, one can take $\hat{M}=  \max_{\inputSpace} \mu_n(\latentVector)$ or, to avoid solving this maximization problem, use the even rougher approximation $\hat{M} = \max_i \{\dataScalar_i\}$. 

Regarding the parameter $L$ note that the definition of Lipschitz continuity in Eq.~(\ref{eq:lipschitz}) does not uniquely identify $L$. In the BO penalization context, small but feasible values of $L$ are preferred, because they produce large exclusion zones and thus more efficient search. Given access to the true objective $f$, one can show that  $L_{\nabla} = \max_{\latentVector \in\mathcal{X}}\|{\nabla \latentForce(\latentVector)}\|$ is a valid Lipschitz constant (see~Supp. Material~\ref{appen:approximation_L} for further details). Note that $L_{\nabla}$ is the smallest value of $L$ that satisfies the Lipschitz condition Eq.~\ref{eq:lipschitz} in the limit $\latentVector_1\to \latentVector_2$ in (\ref{eq:lipschitz}).

\begin{algorithm}[t!]
   \caption{Batch Bayesian Optimization with Local Penalization (BBO-LP).}
   \label{alg:parallel_penalization}
\begin{algorithmic}
   \STATE {\bfseries Input:} dataset $\mathcal{D}_{1} = \{\textbf{x}_i, \dataScalar_i\}_{i=1}^{n}$, batch size $n_b$, iteration budget $m$, acquisition transformation $g$.
  \FOR{$t=1$ {\bfseries to} $m$}
   \STATE Fit a GP to $\mathcal{D}_{t}$.
   \STATE Build the acquisition function $\alpha(\latentVector,\mathcal{I}_{t,0})$ using the current GP.
   \STATE $\tilde{\alpha}_{t,0}(\latentVector) \leftarrow g(\alpha(\latentVector,\mathcal{I}_{t,0}))$.
   \STATE $\hat{L}\leftarrow\max_{\inputSpace} \| \mu_{\nabla}(\latentVector) \|$. 
   \FOR{$j=1$ {\bfseries to} $n_b$}
   \STATE \emph{1. M-step:} $\latentVector_{t,j} \leftarrow \arg \max_{x \in \inputSpace} \left\{\tilde{\alpha}_{t,j-1}(\latentVector)\right\}$.
   \STATE \emph{2. P-step:} $\tilde{\alpha}_{t,j}(\latentVector) \leftarrow \tilde{\alpha}_{t,0}(\latentVector)\prod_{j=1}^{k}\varphi(\latentVector; \latentVector_{t,j},\hat{L})$.
   \ENDFOR
   \STATE $\mathcal{B}_t^{n_b} \leftarrow \{\latentVector_{t,1},\dots,\latentVector_{t,n_b} \}$.
   \STATE $y_{t,1},\dots,y_{t,n_b} \leftarrow$ Parallel evaluations of $f$ at $\mathcal{B}_t^{n_b}$.  
   \STATE $\mathcal{D}_{t+1} \leftarrow \mathcal{D}_{t} \cup  \{( \latentVector_{t,j},y_{t,j})\}_{j=1}^{n_b}$. 
   \ENDFOR
   \STATE Fit GP to $\mathcal{D}_{n}$.
   \STATE \textbf{Returns}: $\hat{\latentVector}_M = \arg \max_{x \in \inputSpace} \left\{\mu(\latentVector)\right\}$.
\end{algorithmic}
\end{algorithm}

We now construct an approximation for $L_{\nabla}$. Assuming that $f$ is a draw from a GP with a (at least) twice differentiable kernel $k$, the gradient of $f$ at $\latentVector^*$ is distributed as a multivariate Gaussian $\nabla \latentForce(\latentVector^*) |\bX, \by,\latentVector^*  \sim \mathcal{N}(\mu_{\nabla}(\latentVector^*),\Sigma_{\nabla}^2(\latentVector^*)) $
with mean vector 
$$\mu_{\nabla}(\latentVector^*) =\partial \bK_{n,*}(\latentVector^*)\tilde{\bK}_{n}^{-1}\by,$$
and covariance matrix 
$$\Sigma^2_{\nabla}(\latentVector^*)=\partial^2 \bK_{*,*}-\partial\bK_{n,*}(\latentVector^*)\tilde{\bK}_{n}^{-1}\partial\bK_{n,*}(\latentVector^*)^\top$$ for  $\tilde{\bK}_n = \bK_{n} + \sigma^2 \textbf{I}$ and where, for $i,j =1,\dots,d$ and $l =1,\dots,n$,
$$(\partial  \bK_{n,*})_{i,l}=  \frac{\partial \bk_N(\latentVector^*) }{\partial x^{(i)}},\quad (\partial^2  \bK_{*,*})_{ij} =\frac{\partial^2 k(\latentVector^*,\latentVector^*) }{\partial x^{(i)}\partial x^{(j)}}.$$
We choose 
$$\hat{L}_{GP-LCA} = \max_{\inputSpace} \| \mu_{\nabla}(\latentVector^*) \| $$
and call this the Gaussian Process Lipschitz Constant Approximation criterion (GP-LCA). Note that this definition of $\hat{L}_{GP-LCA}$ ignores the variance of the gradient, which could be used to identify candidate points to improve the approximation of $L_{\nabla}$ in a Bayesian optimization fashion. The supplement contains further experiments supporting the quality of this approximation. See Algorithm \ref{alg:parallel_penalization} for a description of all the steps described in this section. 

\subsection{Heteroscedastic scenarios}
The use of an unique value of $L$ assumes that the function to optimize is Lipschitz homocedastic. Although this is a typical hypothesis for most BO methods, recent works have pointed out that some real problems do not satisfy this condition \citep{DBLP:journals/corr/AssaelWF14}. It is not the goal of this work to analyze this case further but, interestingly, the method proposed here can be extended to non-Lipschitz cases by replacing $L$ in the penalizers $\varphi(\latentVector; \latentVector_{j},\hat{L})$ by a local values of $L$. For instance, a possible approach would be to replace the local penalizers by $\varphi(\latentVector;\latentVector_j,\hat{L}_j)$ where $\hat{L}_j = \| \mu_{\nabla}(\latentVector_j) \|$.

\subsection{Optimizing the penalized acquisition function}

The optimization of (\ref{eq:penalized_acquisition}) can be performed most easily by any gradient-based method in the log space because there, the gradients have an additive form. More formally, when the transformation used to make the acquisition positive is the \emph{soft-plus} function, $g(z) = \ln(1+e^z)$, the gradient of $\log \tilde{\alpha}_{t,k}(\bx;\mathcal{I}_{t,0})$, being $\tilde{\alpha}_{t,k}(\bx;\mathcal{I}_{t,0})$ the penalized acquisition in Eq.~(\ref{eq:penalized_acquisition}), is:
 \begin{eqnarray}\nonumber
 \nabla \ln  \tilde{\alpha}_{t,k}(\bx, \mathcal{I}_{t,0}) & =& \frac1{\ln(1+e^ {\alpha(\bx; \mathcal{I}_{t,0})})} \frac{e^{\alpha(\bx; \mathcal{I}_{t,0})}} {1+e^{\alpha(\bx; \mathcal{I}_{t,0})}} \cdot\\ \nonumber
 & &  \nabla \alpha(\bx; \mathcal{I}_{t,0}) + \sum_{j=1}^{k-1}\varphi(\bx;\bx_{t,j})^{-1}  \cdot\\ \nonumber
  & & \nabla \varphi(\bx;\bx_{t,j}),\nonumber
 \end{eqnarray}
where $\nabla \alpha(\bx; \mathcal{I}_{t,0})$ is the (assumed known) gradient of the original acquisition function and $\nabla \varphi(\bx;\bx_{t,j})$ are the gradients of the local penalizers
$$\nabla \varphi(\bx;\bx_{t,j}) = \frac{ e^{-z^2}}{\sqrt{2\pi \sigma_n^2(\bx_j)}}\frac{2L}{\|\bx_j-\bx\|}(\bx_j-\bx),$$
See Supp. Materials \ref{appen:optimization_acquisition} for details.

\section{Experimental Section}

\begin{table*}[t]\label{table:UCB}
\begin{center}
\begin{tabular}{cccccccc}
\hline
$d$ & $n_b$ &EI & UCB & Rand-EI & Rand-UCB & SM-UCB & B-UCB\\\hline\hline
\multirow{3}{*}{2} &5&\multirow{3}{*}{0.31$\pm$0.03}&\multirow{3}{*}{0.32$\pm$0.06}&0.32$\pm$0.05&\textbf{0.31$\pm$0.05}&1.86$\pm$1.06&0.56$\pm$0.03\\
&10&&&0.65$\pm$0.32&0.79$\pm$0.42&4.40$\pm$2.97&\textbf{0.59$\pm$0.00}\\
&20&&&0.67$\pm$0.31&0.75$\pm$0.32& - &0.57$\pm$0.01\\
\hline
\multirow{3}{*}{5} &5&\multirow{3}{*}{8.84$\pm$3.69}&\multirow{3}{*}{11.89$\pm$9.44}&9.19$\pm$5.32&10.59$\pm$5.04&137.2$\pm$113.0&\textbf{6.01$\pm$0.00}\\
&10&&&1.74$\pm$1.47&2.20$\pm$1.85&108.7$\pm$74.38&3.77$\pm$0.00\\
&20&&&2.18$\pm$2.30&2.76$\pm$3.06& - &2.53$\pm$0.00\\
\hline
\multirow{3}{*}{10} &5&\multirow{3}{*}{559.1$\pm$1014}&\multirow{3}{*}{1463$\pm$1803}&\textbf{690.5$\pm$947.5}&1825$\pm$2149&9e+04$\pm$7e+04&2098$\pm$0.00\\
&10&&&200.9$\pm$455.9&1149$\pm$1830&9e+04$\pm$1e+05&857.8$\pm$0.00\\
&20&&&639.4$\pm$1204&385.9$\pm$642.9& - &1656$\pm$0.00\\
\hline
\\\hline$d$ & $n_b$ &PE-UCB & Pred-EI & Pred-UCB & qEI & LP-EI & LP-UCB\\\hline\hline
\multirow{3}{*}{2} &5&0.99$\pm$0.74&0.41$\pm$0.15&0.45$\pm$0.16&1.53$\pm$0.86&0.35$\pm$0.11&\textbf{0.31$\pm$0.06}\\
&10&0.66$\pm$0.29&1.16$\pm$0.70&1.26$\pm$0.81&3.82$\pm$2.09&0.66$\pm$0.48&0.69$\pm$0.51\\
&20&0.75$\pm$0.44&1.28$\pm$0.93&1.34$\pm$0.77& - &\textbf{0.50$\pm$0.21}&0.58$\pm$0.21\\
\hline
\multirow{3}{*}{5} &5&123.5$\pm$81.43&10.43$\pm$4.88&11.77$\pm$9.44&15.70$\pm$8.90&11.85$\pm$5.68&10.85$\pm$8.08\\
&10&120.8$\pm$78.56&9.58$\pm$7.85&11.66$\pm$11.48&17.69$\pm$9.04&3.88$\pm$4.15&\textbf{1.88$\pm$2.46}\\
&20&98.60$\pm$82.60&8.58$\pm$8.13&10.86$\pm$10.89&- &6.53$\pm$4.12&\textbf{1.44$\pm$1.93}\\
\hline
\multirow{3}{*}{10} &5&2e+05$\pm$2e+05&793.0$\pm$1226&1412$\pm$3032& - &1881$\pm$1176&1194$\pm$1428\\
&10&6e+04$\pm$8e+04&442.6$\pm$717.9&1725$\pm$3205& - &1042$\pm$1562&\textbf{100.4$\pm$338.7}\\
&20&5e+04$\pm$4e+04&1091$\pm$1724&2231$\pm$3110& - &1249$\pm$1570& \textbf{20.75$\pm$50.12}\\
\hline
\end{tabular}
\caption{Results for the gSobol function across different dimensions, batch sizes and methods. For each algorithm, the mean and standard deviation are shown. Best results among the batch methods are highlighted in bold. `-' represents that the method could not complete the first iteration within the time budget. The value of $f$ at the minimum is always zero. EI and UCB represent the Expected improvement and the upper confidence bound acquisitions. Rand stands for the random batch collection. SM is the simulating and matching approach. Pred is the predictive approach and LP the local penalization method presented in this work. qEI is the multi-point expected improvement. }
\end{center}
\end{table*}

This section compares the performance of Algorithm \ref{alg:parallel_penalization} with the state-of-the-art methods for batch BO. We label the different methods by means of the batch design type followed by the acquisition used: {Rand} is used when the first element in the batch is collected maximizing the acquisition and the remaining ones randomly, {B} and {PE} denote the exploratory approaches in \citep{schonlau} and \citep{journals/corr/abs-1304-5350}, {Pred} is used in cases when the model is used to generate `fake' batch observations as in \citep{DBLP:conf/icml/AzimiJF12}, {SM} identifies the simulating and matching method  \citep{DBLP:conf/nips/AzimiFF10} and {LP} stands for our local penalization method. The multi-point expected improvement \citep{ChevalierG13} is denoted by {qEI}. Two acquisition functions are used: the expected improvement (EI) defined as $\alpha_{EI}(\latentVector;\I_n) = \left(u \Phi(u) + \phi (u) \right)\sigma_n(\latentVector),$
where $u = (\mu_n(\latentVector) - y_{\min})/\sigma_n(\latentVector)$ and $\Phi(\cdot)$,  $\phi(\cdot)$ are the standard Gaussian distribution and density functions respectively and $y_{\min}$ is the best current location and the Upper Confidence Bound (UCB) defined as $\alpha_{UCB}(\latentVector;\I_n) = \mu_n(\latentVector) +\kappa \sigma_n(\latentVector)$, with $\kappa\geq 0$. The batch methods that can be used with an arbitrary acquisition function are tested using both, with the exception of the SM whose implementation is only available with the UCB. When used in a sequential setting (for baseline reference) the EI and UCB are referred by their acronyms. In total, we use 2 sequential and 10 batch methods. To run the B, PE, SM, methods, we use the available \textsc{Matlab} code.\footnote{http://econtal.perso.math.cnrs.fr/software/. Note that an alternative implementation of the GP-B-UCB code is available at http://www.its.caltech.edu/~tadesaut/GPBUCBCode/ but we used the former one for consistency in the comparisons.} The implementation of these methods optimize $f$ by searching its optimum in a fine grid, which is an advantage computationally but a drawback in terms of precision. The qEI was taken from the R-package DiceOptim.\footnote{http://cran.r-project.org/web/packages/DiceOptim} Unless specified otherwise, the default implemented settings of all the previous methods are used. 

We perform: (i) a simulation in which the performance of the algorithms is compared for a fixed time budget across different problem dimensions, batch sizes and acquisition functions and (ii) a comparison of the \emph{gained information per second rate} in three objective functions with different evaluation costs. We always minimize the objective, minimizing $-f$ in examples in which the goal is to find maximum of $f$. In all the experiments the exponentiated quadratic (EQ) covariance $k(\latentVector,\latentVector') = \theta \exp(- \gamma \|\latentVector-\latentVector'\|^2 )$, $\theta,\gamma >0$ is used in the GP, whose parameters are optimized by maximizing the marginal likelihood from the best of 10 random initializations. The results are taken over 20 replicates with different initial values. All the simulations were done on Amazon EC2 servers with Intel Xeon E5-2666 processors and 2 virtual CPUs except the SVR tuning with 16 virtual CPUs. 
\subsection{Comparisons in terms of the dimension, batch size and acquisition function}\label{sec:gSobol}

 \begin{figure*}[t!]
  \begin{center}\label{figure:comparisons}
   \mbox{
    \subfigure[Results for the Cosines function - batch iterations.]{\includegraphics[width=0.44\linewidth]
{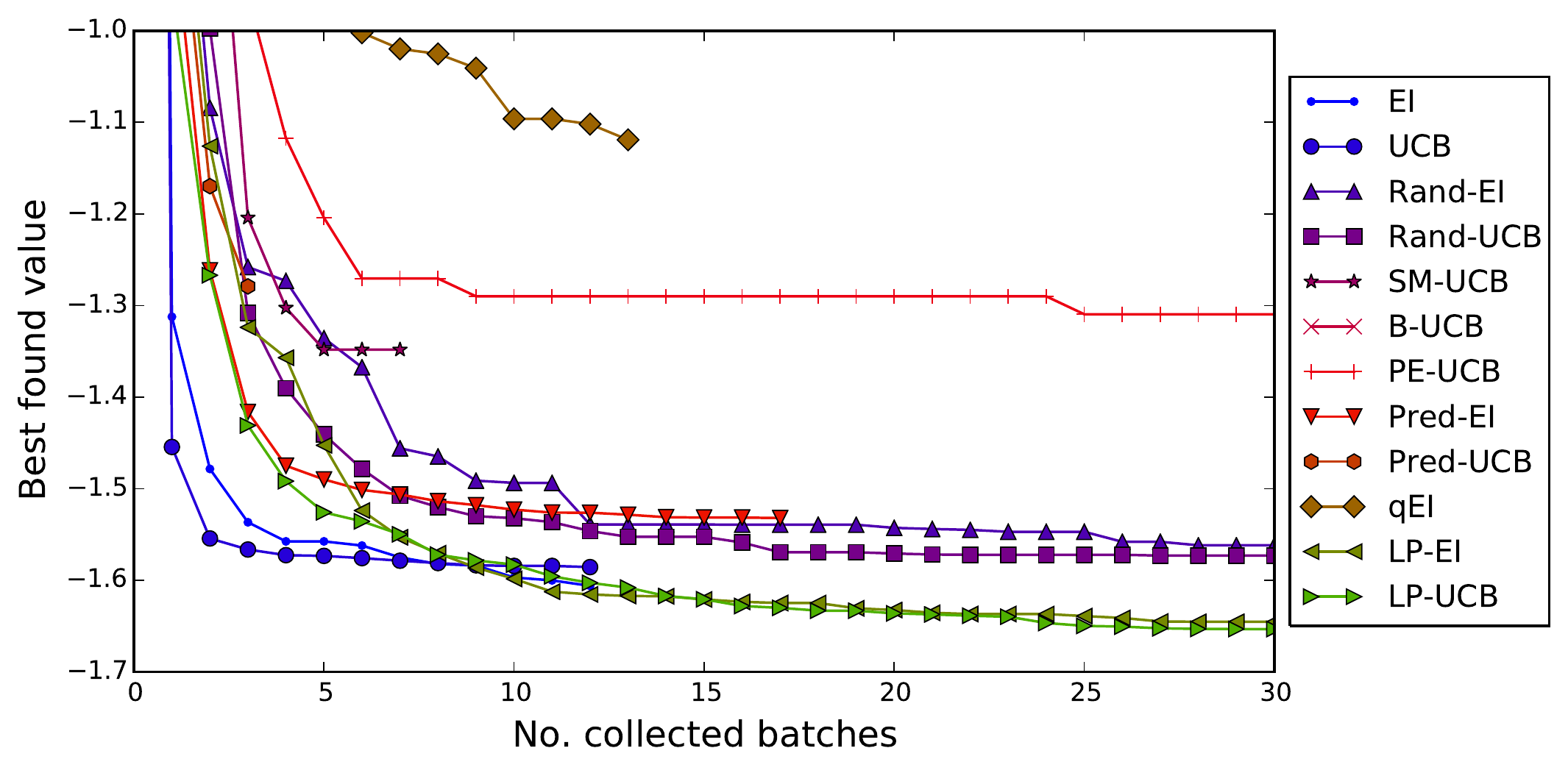}} \quad
    \subfigure[Results for the Cosines function - running time.]{\includegraphics[width=0.44\linewidth]
{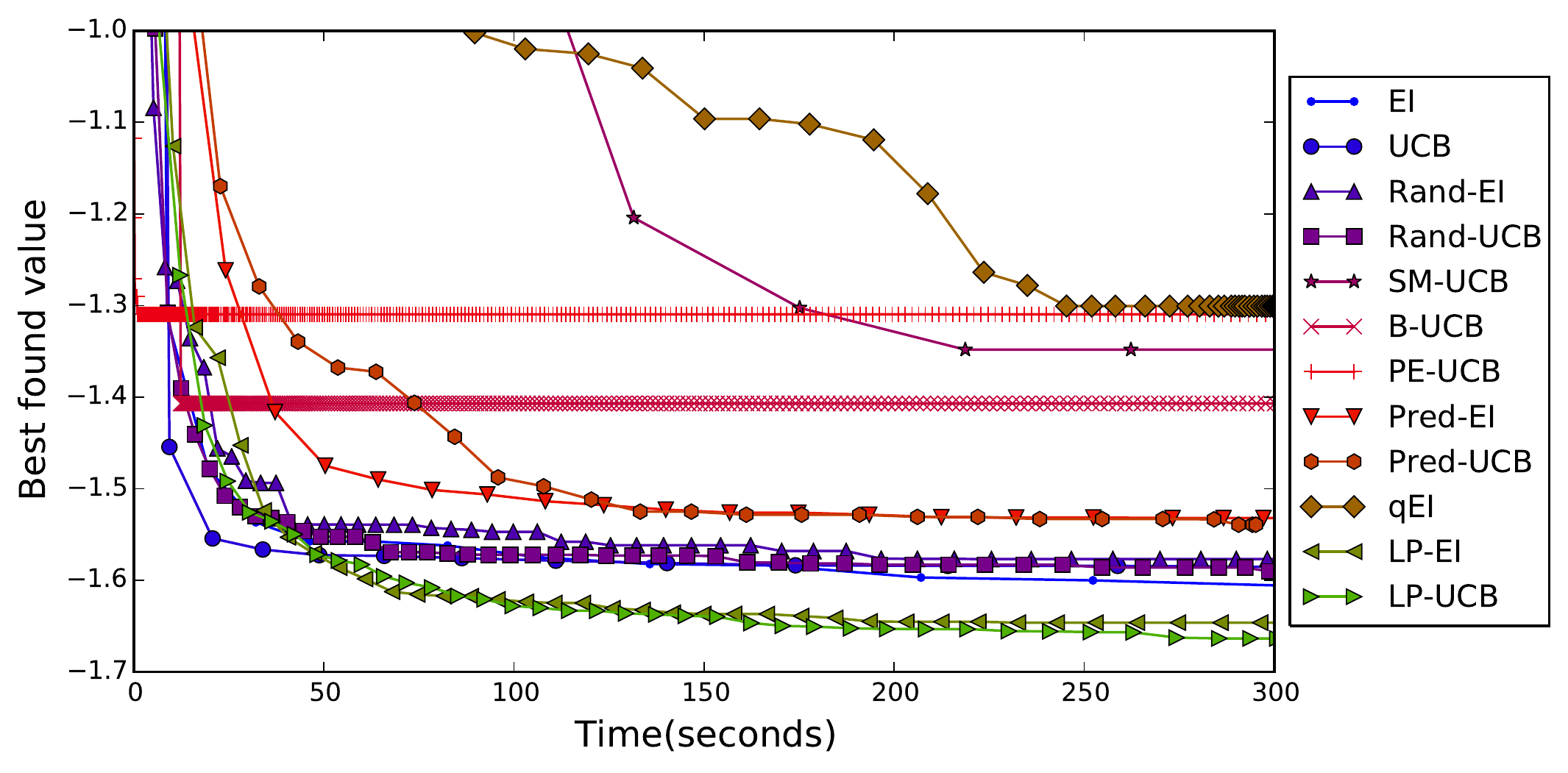}} \quad
 }
   \mbox{
    \subfigure[Results for the wet-lab function - batch iterations.]{\includegraphics[width=0.44\linewidth]
{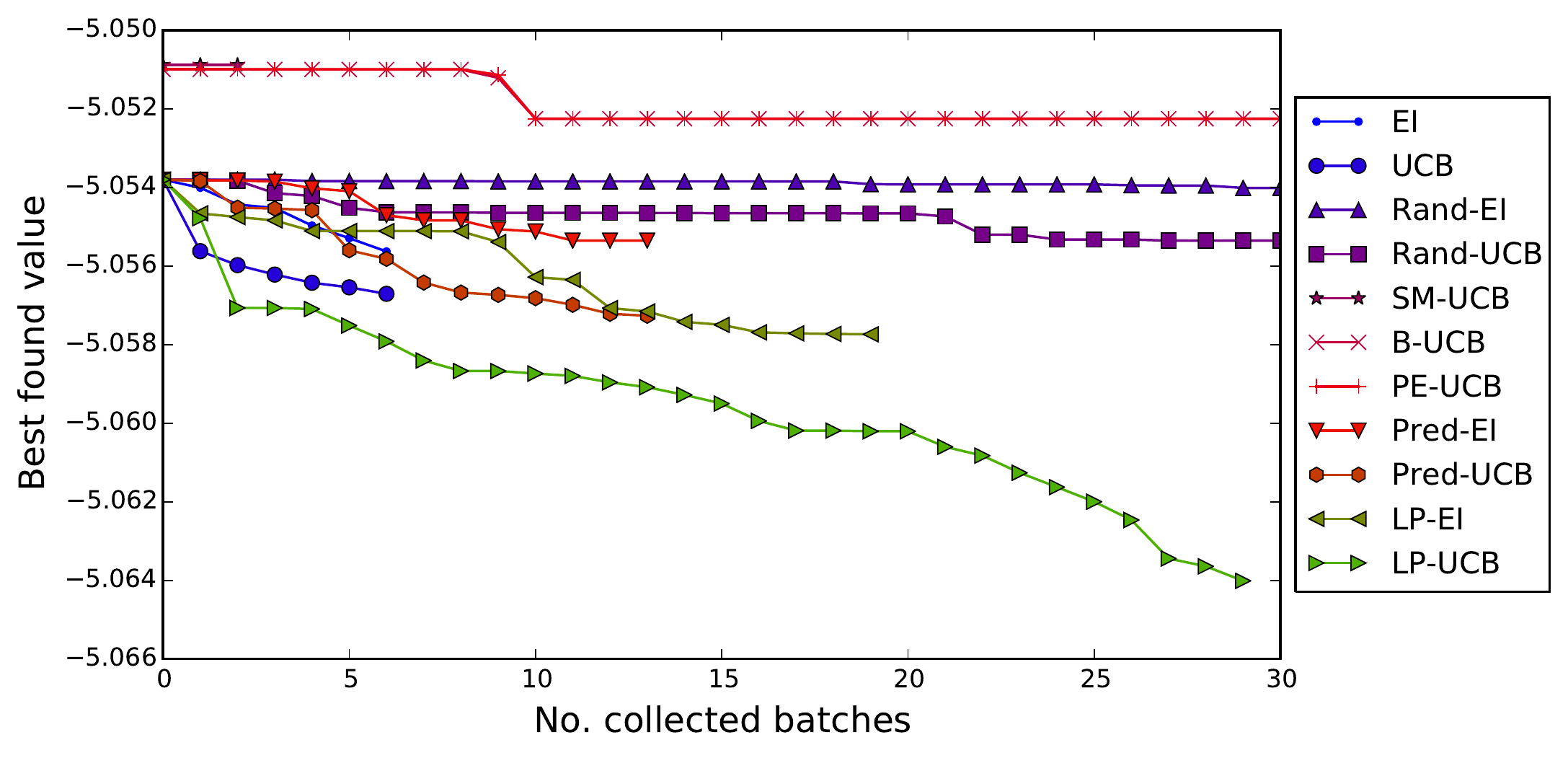}} \quad
    \subfigure[Results for the wet-lab function - running time.]{\includegraphics[width=0.44\linewidth]
{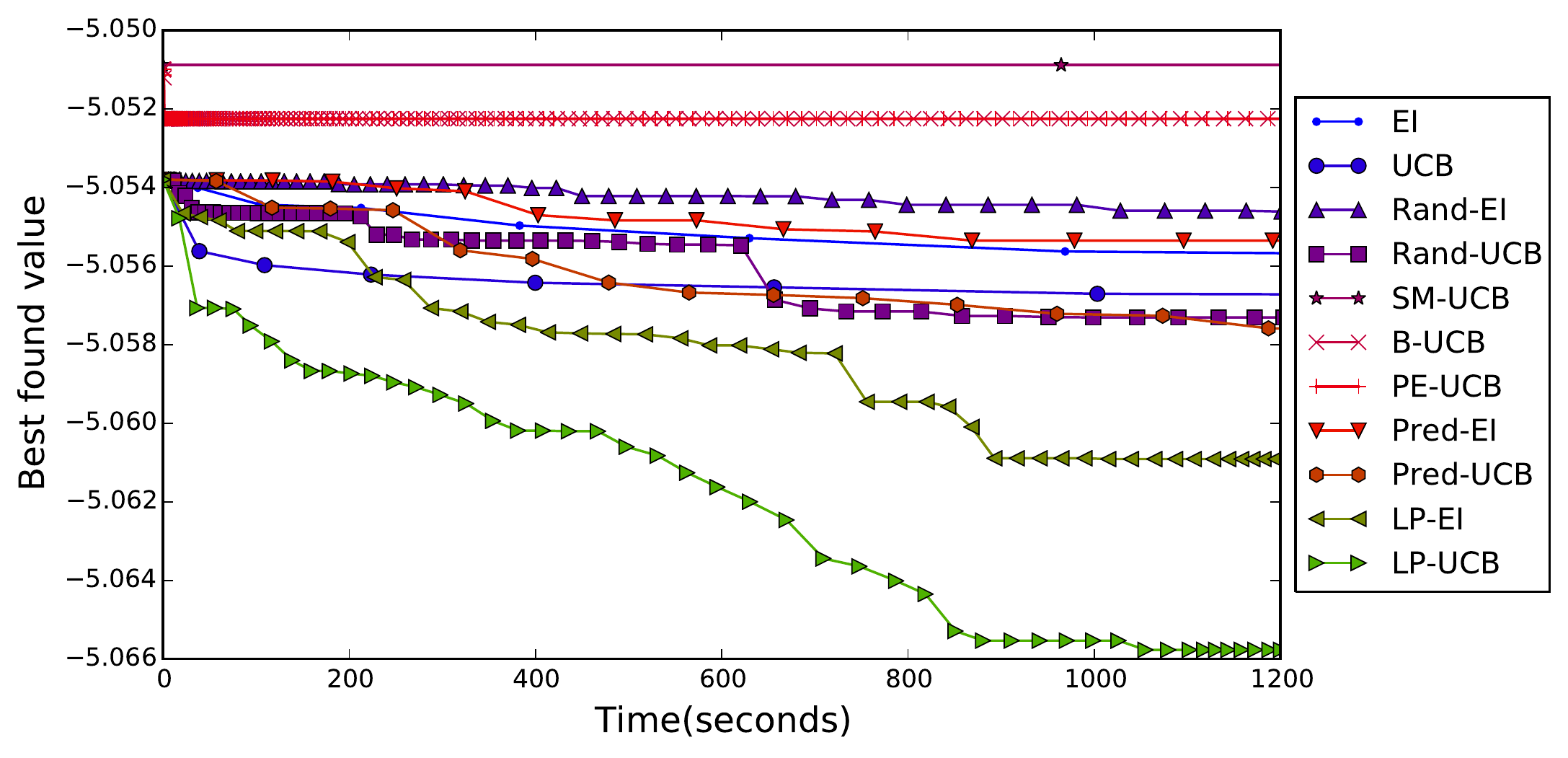}} \quad
  }
   \mbox{
      \subfigure[Results for the SVR function - batch iterations.]{\includegraphics[width=0.44\linewidth]
{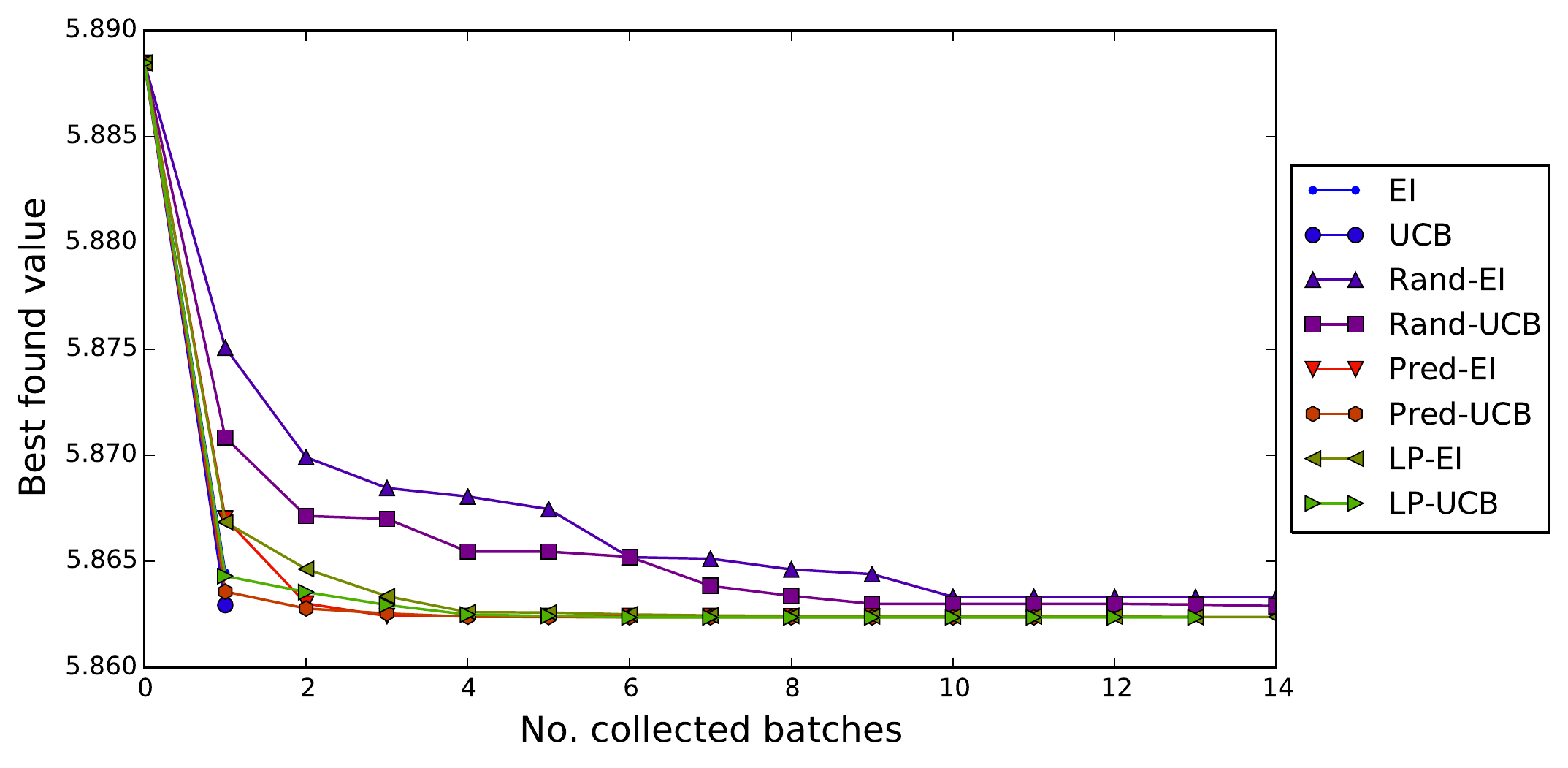}} \quad
     \subfigure[Results for the SVR function - running time.]{\includegraphics[width=0.44\linewidth]
{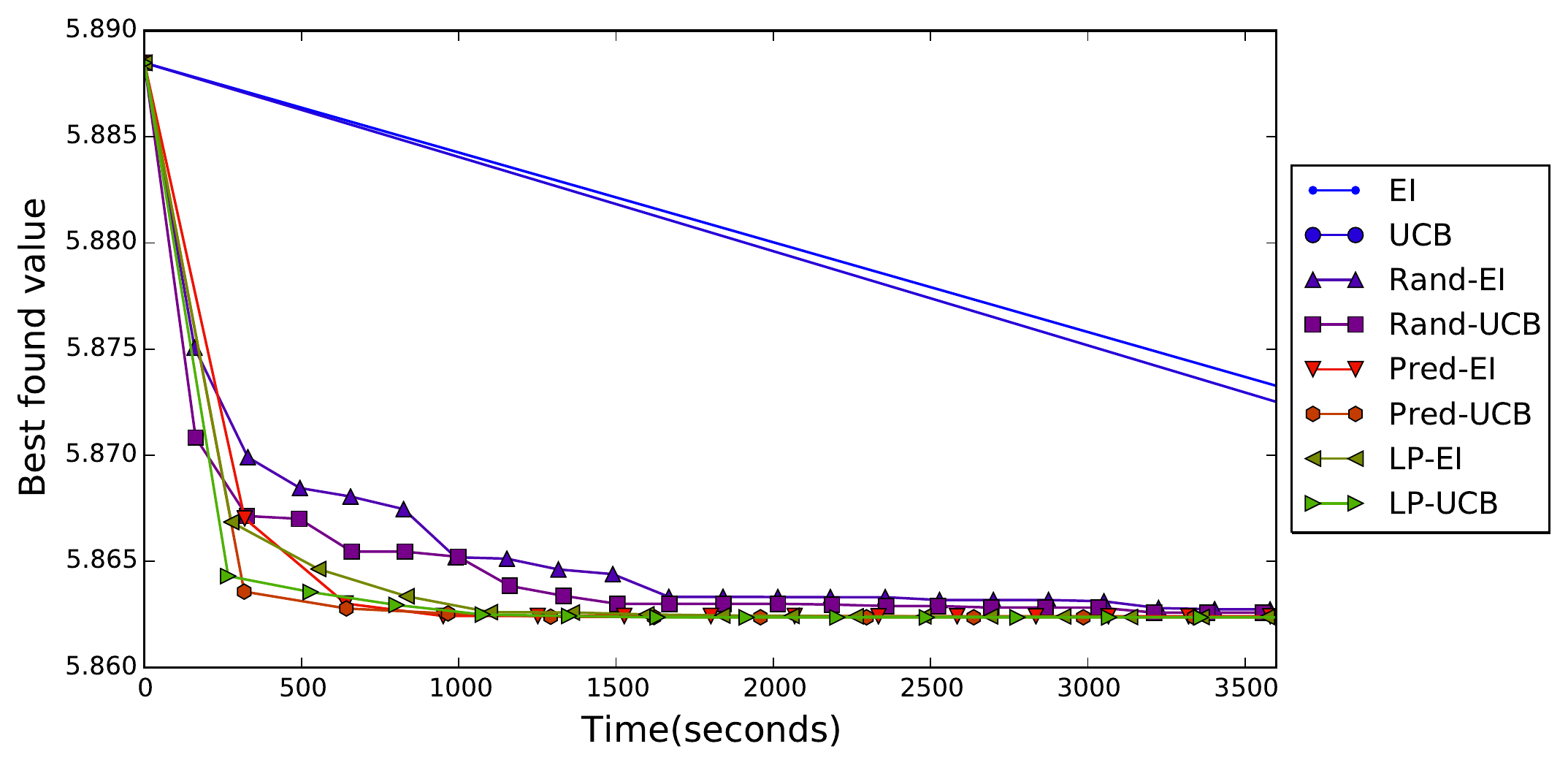}}    }
 \end{center}
\caption{Results for the three test functions described in Section \ref{sec:simulated_functions}. Geometric figures on top of the lines represent the moments in which the batches are evaluated. See caption of Table \ref{table:UCB} for details on the acronyms.}\label{figure:example}
\end{figure*}

We consider the gSobol function (see Supp. Materials \ref{sec:functions}) to compare the above mentioned methods across dimensions $d=2,5,10$ and batch sizes, $n_b = 5,10, 20$. For methods using the UCB, $\kappa$ was fixed to 2, which allows us to compare the different batch designs using the same acquisition function. For dimension 2, 5 and 10, we use a time budget of 1, 5 and 10 mins. respectively. Table \ref{table:UCB}  shows the averaged best value found by each algorithm for all the iterations completed within the time limit. In general, the batch methods using the UCB show a better performance that methods using the EI, especially in dimensions 5 and 10. The overall best technique is the LP-UCB, that achieves the best results in 5 of the 9 cases. It is also notable that it exhibits fairly small standard deviations compared with the rest of the methods and it is coherent accumulating information about the optimum of $f$ in terms of the batch size: as $n_b$ increases the results are consistently better. In dimension 2 and batch size 5, the LP-EI is the best method. There are three cases in which the LP batch designs are not the most competitive (although still providing good results). Exploratory approaches works well in low dimensional cases, being the B-UCB the best method in two scenarios.

\subsection{Comparisons in terms of the cost to evaluate the objective}\label{sec:simulated_functions}

We choose three scenarios to compare the algorithms in terms of the running time. The examples correspond to three functions that are cheap, moderate and expensive to evaluate. More specifically, the first experiment uses a function (Cosines) that is inexpensive to evaluate but quite multi-modal. The second experiment is motivated by a wet-lab experimental design. We work with a surface that emulates the performance of mammalian cells in protein production given different gene designs. The function has dimension 71 and is is moderately expensive to evaluate since it corresponds to the predictive mean of a GP trained over 1,500 data instances. The qEI was not used in this experiment due to the huge computational effort required to jointly optimize the batches in dimension 71. The third experiment involves the tuning of the three parameters of a support vector regression (SVR) \citep{Drucker1997Support} in a example with 45730 instances and 9 continuous  attributes \citep{Bache+Lichman:2013}. The objective function is the cross-validation error of the model, which is expensive to evaluate due to the amount of data used. See Supp. Materials \ref{sec:functions} for further details. We take a batch size of $n_b=5$ for the Cosines function and $n_b=10$ for the wet-lab and SVR experiments. We compare the averaged best found results in terms of the number of collected batches and the wall-clock time. In the last experiment we use the SVR implementation available in scikit--learn\footnote{http://scikit-learn.org/stable/index.html.} and only the methods implemented in python are used (EI, UCB, Rand-EI, Rand-UCB, Pred-EI, Pred-UCB, LP-EI and LP-UCB).

In the Cosines experiment both the sequential EI and UCB policies achieve the best results during the first 10 iterations of the algorithms (2 full batches). As the algorithms progress, however, a significant improvement is observed by the LP-EI and LP-UCB methods in  terms of the number iterations and in terms of the wall-clock-time. When many points are collected, the update of the GP is more expensive and a good batch design is able to explore regions that the sequential method cannot. The rest of the batch methods, however, are not able to do this exploration efficiently, which leads to poorer results. Similar results are obtained for the wet-lab experiment. The LP-EI and LP-UCB are again the most competitive techniques improving the rest of the batch methods and the sequential policies. The differences are even more significant in this scenario.  Since $f$ is now more expensive to evaluate, the parallelization of the evaluations makes the search much more efficient, specially for the LP-UCB method. Regarding the last experiment, the cost of evaluating the function dominates the cost of designing the batch. In this case the performance of the different batch methods is comparable but significantly better than the sequential policies due to the parallel evaluations of $f$. The results for the three functions  are coherent with those observed in Section \ref{sec:gSobol} showing that the BBP-LP methods is overall the most efficient method for batches collection in BO.

\section{Discussion}
We have investigated a new heuristic for batch BO, BBO-LP, that significantly reduces the computational burden of non-parallelizable tasks. The resulting method can be used with any acquisition function and it is able to make fast and appropriate decisions about the locations where $f$ should be evaluated. When the batch evaluations of $f$ are parallelizable this is an important advantage, meaning that they don't lead to considerable additional computational overhead. We have found other interesting results. In simple scenarios, batch policies based on random exploration work reasonably well in terms of the information gained per second. When the complexity of the problem increases, however, methods that make use of some information about $f$  improve the random policy. In particular, the approach here proposed makes use of the Lipschitz continuity of $f$ to model the interaction between the elements in the batch. In spirit, this is similar to use the GP to predict the evaluations of $f$ but, in practice, is much more efficient because it avoids the re-computation of the GP after every point is selected. The limitations of this approach are, however, determined by the ability to learn correctly a small enough, and valid, Lipschitz constant for $f$.

 One could also wonder whether it is necessary to require that sample paths from the GP measure on $f$ should be Lipschitz-continuous themselves. This would severely restrict the applicability of this notion, because the relationship between regularity of the kernel and the sample paths is complicated. Even if the kernel is Lipschitz-continuous, sample paths may not be Lipschitz \citep{Adler1981}. However, our approach only tries to model the effect of evaluations on the BO objective, not the GP probability measure itself. Many BO objectives, in particular the EI and UCB, are smooth functions  of only the sufficient statistics (mean and covariance function) of the GP posterior. Both the posterior mean and covariance function are members of the Reproducing kernel Hilbert space induced by the kernel (i.e. they are weighted sums of kernel functions). Thus, if the kernel is Lipschitz, so is the acquisition function, even if the GP measure itself has non-Lipschitz sample paths. Finally, note that our local repulsion criterion naturally suggests a Latin square design for the case when no functional values have been been acquired. The latin square design is widely suggested for this domain \citep{Jones:1998:EGO:596070.596218}.



\newpage
\bibliographystyle{plainnat}
\bibliography{bobatch_nips,../../../bib/lawrence,../../../bib/other,../../../bib/zbooks}

\begin{thebibliography}{27}
\providecommand{\natexlab}[1]{#1}
\providecommand{\url}[1]{\texttt{#1}}
\expandafter\ifx\csname urlstyle\endcsname\relax
  \providecommand{\doi}[1]{doi: #1}\else
  \providecommand{\doi}{doi: \begingroup \urlstyle{rm}\Url}\fi

\bibitem[Adler(1981)]{Adler1981}
Robert~J. Adler.
\newblock \emph{The geometry of random fields}.
\newblock Wiley, 1981.

\bibitem[Assael et~al.(2014)Assael, Wang, and
  de~Freitas]{DBLP:journals/corr/AssaelWF14}
John{-}Alexander~M. Assael, Ziyu Wang, and Nando de~Freitas.
\newblock Heteroscedastic treed bayesian optimisation.
\newblock \emph{CoRR}, abs/1410.7172, 2014.

\bibitem[Azimi et~al.(2010)Azimi, Fern, and Fern]{DBLP:conf/nips/AzimiFF10}
Javad Azimi, Alan Fern, and Xiaoli Fern.
\newblock Batch {B}ayesian optimization via simulation matching.
\newblock In \emph{Advances in Neural Information Processing Systems}, pages
  109--117, 2010.

\bibitem[Azimi et~al.(2011)Azimi, Jalali, and
  Fern]{DBLP:journals/corr/abs-1110-3347}
Javad Azimi, Ali Jalali, and Xiaoli Fern.
\newblock Dynamic batch {B}ayesian optimization.
\newblock \emph{CoRR}, abs/1110.3347, 2011.

\bibitem[Azimi et~al.(2012)Azimi, Jalali, and Fern]{DBLP:conf/icml/AzimiJF12}
Javad Azimi, Ali Jalali, and Xiaoli~Zhang Fern.
\newblock Hybrid batch {B}ayesian optimization.
\newblock In \emph{Proceedings of the 29th International Conference on Machine
  Learning}, 2012.

\bibitem[Bache and Lichman(2013)]{Bache+Lichman:2013}
Kevin Bache and Moshe Lichman.
\newblock {UCI} machine learning repository, 2013.

\bibitem[Bergstra et~al.(2011)Bergstra, Bardenet, Bengio, and
  K{\'{e}}gl]{Bergstra+al-NIPS2011}
James Bergstra, R{\'{e}}my Bardenet, Yoshua Bengio, and Bal{\'{a}}zs
  K{\'{e}}gl.
\newblock Algorithms for hyper-parameter optimization.
\newblock In \emph{NIPS'2011}, 2011.

\bibitem[Chevalier and Ginsbourger(2013)]{ChevalierG13}
Clément Chevalier and David Ginsbourger.
\newblock Fast computation of the multi-points expected improvement with
  applications in batch selection.
\newblock In Giuseppe Nicosia and Panos~M. Pardalos, editors, \emph{LION},
  volume 7997 of \emph{LNCS}, pages 59--69. Springer, 2013.
\newblock ISBN 978-3-642-44972-7.

\bibitem[Contal et~al.(2013)Contal, Buffoni, Robicquet, and
  Vayatis]{journals/corr/abs-1304-5350}
Emile Contal, David Buffoni, Alexandre Robicquet, and Nicolas Vayatis.
\newblock Parallel {G}aussian process optimization with upper confidence bound
  and pure exploration.
\newblock \emph{CoRR}, abs/1304.5350, 2013.

\bibitem[Desautels et~al.(2012)Desautels, Krause, and
  Burdick]{DBLP:conf/icml/DesautelsKB12}
Thomas Desautels, Andreas Krause, and Joel~W. Burdick.
\newblock Parallelizing exploration-exploitation tradeoffs with {G}aussian
  process bandit optimization.
\newblock In \emph{Proceedings of the 29th International Conference on Machine
  Learning}, 2012.

\bibitem[Drucker et~al.(1997)Drucker, Chris, Kaufman, Smola, and
  Vapnik]{Drucker1997Support}
Harris Drucker, Chris, Burges~L. Kaufman, Alex Smola, and Vladimir Vapnik.
\newblock Support vector regression machines.
\newblock In \emph{Advances in Neural Information Processing Systems 9}, pages
  155--161, 1997.

\bibitem[Floudas and Pardalos(2009)]{reference/opt/2009}
Christodoulos~A. Floudas and Panos~M. Pardalos, editors.
\newblock \emph{Encyclopedia of Optimization, Second Edition}.
\newblock Springer, 2009.

\bibitem[Frazier(2012)]{Frazier2012}
P.~I. Frazier.
\newblock Parallel global optimization using an improved multi-points expected
  improvement criterion.
\newblock In \emph{INFORMS Optimization Society Conference, Miami FL}, 2012.

\bibitem[Ginsbourger et~al.(2011)Ginsbourger, Janusevskis, and
  Le~Riche]{ginsbourger:hal-00507632}
David Ginsbourger, Janis Janusevskis, and Rodolphe Le~Riche.
\newblock {Dealing with asynchronicity in parallel Gaussian Process based
  global optimization}.
\newblock Technical report, 2011.

\bibitem[Gonz\'alez et~al.(2014)Gonz\'alez, Longworth, James, and
  Lawrence]{gonzalez2014}
Javier Gonz\'alez, Joseph Longworth, David James, and Neil Lawrence.
\newblock Bayesian optimisation for synthetic gene design.
\newblock \emph{NIPS Workshop on Bayesian Optimization in Academia and
  Industry}, 2014.

\bibitem[Hennig and Schuler(2012)]{HennigSchuler2012}
Philipp Hennig and Christian~J. Schuler.
\newblock Entropy search for information-efficient global optimization.
\newblock \emph{Journal of Machine Learning Research}, 13, 2012.

\bibitem[Hern\'andez-Lobato et~al.(2014)Hern\'andez-Lobato, Hoffman, and
  Ghahramani]{NIPS2014_5324}
Jos\'e~M. Hern\'andez-Lobato, Matthew~W. Hoffman, and Zoubin Ghahramani.
\newblock Predictive entropy search for efficient global optimization of
  black-box functions.
\newblock In \emph{Advances in Neural Information Processing Systems 27}, pages
  918--926. Curran Associates, Inc., 2014.

\bibitem[Horst and Pardalos(1995)]{opac-b1088635}
Reiner Horst and Panos~M. Pardalos, editors.
\newblock \emph{Handbook of global optimization}.
\newblock Nonconvex optimization and its applications. Kluwer Academic
  Publishers, Dordrecht, Boston, 1995.

\bibitem[Jalali et~al.(2013)Jalali, Azimi, Fern, and
  Zhang]{DBLP:conf/pkdd/JalaliAFZ13}
Ali Jalali, Javad Azimi, Xiaoli Fern, and Ruofei Zhang.
\newblock A lipschitz exploration-exploitation scheme for {B}ayesian
  optimization.
\newblock In \emph{Machine Learning and Knowledge Discovery in Databases},
  pages 210--224, 2013.

\bibitem[Janusevskis et~al.(2012)Janusevskis, Riche, Ginsbourger, and
  Girdziusas]{conf/lion/JanusevskisRGG12}
Janis Janusevskis, Rodolphe~Le Riche, David Ginsbourger, and Ramunas
  Girdziusas.
\newblock Expected improvements for the asynchronous parallel global
  optimization of expensive functions: Potentials and challenges.
\newblock In Y.~Hamadi and M.~Schoenauer, editors, \emph{LION}, volume 7219 of
  \emph{LNCS}, pages 413--418. Springer, 2012.

\bibitem[Jones et~al.(1998)Jones, Schonlau, and
  Welch]{Jones:1998:EGO:596070.596218}
Donald~R. Jones, Matthias Schonlau, and William~J. Welch.
\newblock Efficient global optimization of expensive black-box functions.
\newblock \emph{Journal of Global Optimization}, 13\penalty0 (4):\penalty0
  455--492, 1998.

\bibitem[Osborne(2010)]{osborne_bayesian_2010}
Michael Osborne.
\newblock \emph{Bayesian {Gaussian} {Processes} for {Sequential} {Prediction},
  {Optimisation} and {Quadrature}}.
\newblock PhD thesis, PhD thesis, University of Oxford, 2010.

\bibitem[Očenášek and Schwarz(2000)]{firstPBO}
Jiří Očenášek and Josef Schwarz.
\newblock The parallel {B}ayesian optimization algorithm.
\newblock In Peter Sinčák, Ján Vaščák, Vladimír Kvasnička, and Radko
  Mesiar, editors, \emph{The State of the Art in Computational Intelligence},
  volume~5 of \emph{Advances in Soft Computing}, pages 61--67. Physica-Verlag
  HD, 2000.

\bibitem[Rasmussen and Williams(2005)]{Rasmussen:2005:GPM:1162254}
Carl~Edward Rasmussen and Christopher K.~I. Williams.
\newblock \emph{Gaussian Processes for Machine Learning (Adaptive Computation
  and Machine Learning)}.
\newblock The MIT Press, 2005.

\bibitem[Schonlau et~al.(1998)Schonlau, Welch, and Jones]{schonlau}
Matthias Schonlau, William~J. Welch, and Donald~R. Jones.
\newblock \emph{Global versus local search in constrained optimization of
  computer models}, volume Volume 34 of \emph{Lecture Notes--Monograph Series},
  pages 11--25.
\newblock Institute of Mathematical Statistics, Hayward, CA, 1998.

\bibitem[Snoek et~al.(2012)Snoek, Larochelle, and Adams]{snoek68}
Jasper Snoek, Hugo Larochelle, and Ryan~P. Adams.
\newblock Practical {B}ayesian optimization of machine learning algorithms.
\newblock In \emph{Advances in Neural Information Processing Systems 25}, pages
  2960--2968, 2012.

\bibitem[Strongin and Sergeyev(2000)]{opac-b1107684}
Roman~G. Strongin and Yaroslav~D. Sergeyev.
\newblock \emph{Global optimization with non-convex constraints : sequential
  and parallel algorithms}.
\newblock Nonconvex Optimization and Its Applications. Kluwer academic
  publishers, Dordrecht, Boston, Londres, 2000.

\end{thebibliography}

\newpage
$\mbox{}$
\newpage

\appendix

\section{Proof of Proposition 1}\label{appen:proof}

We compute the explicit form of the penalization functions $\varphi(\bx;\bx_j)$. The distribution of $r_j$ is Gaussian with mean $(M-\mu_n(\bx_j))/L$ and variance $ \sigma_n^2(\bx_j)/L^2$ by the properties of $f(\bx_j)$.  Then we obtain that
\begin{eqnarray}\nonumber
\varphi(\bx;\bx_j)  & = &  1- p (\bx  \in B_{r_j}(\bx_j) ) \\ \nonumber
& =  & 1- p(r_j \geq \| \bx_j - \bx\|_p)\\ \nonumber
& =  & p(r_j \leq \| \bx_j - \bx\|_p)\\ \nonumber
& =  & p\left(\mathcal{N}(0,1) \leq \frac{L \|\bx_j-\bx\|_p-M+ \mu_n(\bx_j)}{\sigma_n(\bx_j)} \right)\\ \nonumber
& = & \Phi\left(\frac{L \|\bx_j-\bx\|_p-M+ \mu_n(\bx_j)}{\sigma_n(\bx_j)} \right)\\ \nonumber
& = &  \frac{1}{2} \mbox{erfc} \left(-z \right) \\ \nonumber
\end{eqnarray}
for
$$z = \frac1{\sqrt{2\sigma_n^2(\bx_j)}} \left(L\|\bx_j-\bx \| - M + \mu_n(\bx_j) \right).$$

\section{Optimization of the penalized acquisition function}\label{appen:optimization_acquisition}

Under the proposed local penalization method, to select the $k$-th element of the $t$-th batch  requires the optimization of the function
$$\tilde{\alpha}_{t,k}(\bx;\mathcal{I}_{t,0}) =g(\alpha(\bx; \mathcal{I}_{t,0}))\prod_{j=1}^{k-1}\varphi(\bx;\bx_{t,j}), $$
which can be done by any gradient descend method as follows. We fist map the problem into  the natural log space by observing that
$$\arg \max_{x \in \mathcal{X}} \left\{\tilde{\alpha}_{t,k}(\bx, \mathcal{I}_{t,0})\right\} = \arg \max_{x \in \mathcal{X}} \left\{\ln \tilde{\alpha}_{t,k}(\bx, \mathcal{I}_{t,0})\right\}. $$
Applying the properties of the logarithms we transform the problem into the maximization of
$$ \ln  \tilde{\alpha}_{t,k}(\bx, \mathcal{I}_{t,0}) = \ln \left[ g(\alpha(\bx; \mathcal{I}_{t,0})) \right] + \sum_{j=1}^{k-1} \ln \left[\varphi(\bx;\bx_{t,j})\right].$$
The gradient with respect to $\bx$ is now easy to calculate since the problem is in additive form. First, note that the gradients of the local penalizers $\nabla \varphi(\bx;\bx_{t,j})$ are
\begin{equation}\nonumber
\nabla \varphi(\bx;\bx_{t,j}) = \frac{ e^{-z^2}}{\sqrt{2\pi \sigma_n^2(\bx_j)}}\frac{2L}{\|\bx_j-\bx\|}(\bx_j-\bx),
\end{equation}
with 
$$z = \frac1{\sqrt{2\sigma_n^2(\bx_j)}} \left(L\|\bx_j-\bx \| - M + \mu_n(\bx_j) \right).$$
Then, it holds that
\begin{eqnarray}\nonumber
\nabla \ln  \tilde{\alpha}_{t,k}(\bx, \mathcal{I}_{t,0}) & =  & [ g(\alpha(\bx; \mathcal{I}_{t,0}))^{-1} \\\nonumber
& & \frac{d }{d\alpha(\bx; \mathcal{I}_{t,0})}g(\alpha(\bx; \mathcal{I}_{t,0})) ] \nabla\alpha(\bx; \mathcal{I}_{t,0})\\ \nonumber
& + &\sum_{j=1}^{k-1}\varphi(\bx;\bx_{t,j})^{-1} \nabla \varphi(\bx;\bx_{t,j})\nonumber
\end{eqnarray}
where $\nabla \alpha(\bx; \mathcal{I}_{t,0})$ is the (assumed known) gradient of the original acquisition function. In cases in which the acquisition is already positive it is natural to choose $g(z) = z$ and the gradient of $ \ln  \tilde{\alpha}_{t,k}(\bx, \mathcal{I}_{t,0})$ reduces to
\begin{eqnarray}\nonumber
\nabla \ln \tilde{ \alpha}_{t,k}(\bx, \mathcal{I}_{t,0}) &= & \alpha(\bx; \mathcal{I}_{t,0})^{-1} \nabla \alpha(\bx; \mathcal{I}_{t,0}) + \\ \nonumber
& =&\sum_{j=1}^{k-1}\varphi(\bx;\bx_{t,j})^{-1} \nabla \varphi(\bx;\bx_{t,j}).
\end{eqnarray}
 When $\alpha(\bx; \mathcal{I}_{t,0})$ is not necessarily positive one can take $g(z) = \exp(z)$ and the gradient simplifies to
 $$ \nabla \ln  \tilde{\alpha}_{t,k}(\bx, \mathcal{I}_{t,0}) = \nabla \alpha(\bx; \mathcal{I}_{t,0}) + \sum_{j=1}^{k-1}\varphi(\bx;\bx_{t,j})^{-1} \nabla \varphi(\bx;\bx_{t,j}).$$
When $g(z) = \ln(1+e^z)$ the gradient is
 \begin{eqnarray}\nonumber
 \nabla \ln  \tilde{\alpha}_{t,k}(\bx, \mathcal{I}_{t,0}) & =& \frac1{\ln(1+e^ {\alpha(\bx; \mathcal{I}_{t,0})})} \frac{e^{\alpha(\bx; \mathcal{I}_{t,0})}} {1+e^{\alpha(\bx; \mathcal{I}_{t,0})}} \cdot\\ \nonumber
 & &  \nabla \alpha(\bx; \mathcal{I}_{t,0}) + \sum_{j=1}^{k-1}\varphi(\bx;\bx_{t,j})^{-1}  \cdot\\ \nonumber
  & & \nabla \varphi(\bx;\bx_{t,j}).\nonumber
 \end{eqnarray}

\section{Lipschitz constant approximation}\label{appen:approximation_L}
In this section we elaborate in the approximation of the Lipschitz constant. First, we include the following proposition that allows to uniquely identify a valid value of $L$. 

\begin{proposition}
Let $f: {\mathcal X} \to \bbbr$ be a L-Lipschitz continuous function defined on a compact subset ${\mathcal X} \subseteq \bbbr^d$. Take 
$$L_p = \max_{\bx \in\mathcal{X}}\|{\nabla f(\bx)}\|_p,$$
where $\nabla f(\bx) = \left(\frac{\partial f}{\partial \bx_1},\cdots, \frac{\partial f}{\partial \bx_p} \right)^\top$.
Then, $L_p$ is a valid Lipschitz constant such that the Lipschitz condition  
$$|f(\bx_1) - f(\bx_2) | \leq L_p \|\bx_1 -\bx_2 \|_q,$$
where  $\frac{1}{s} + \frac{1}{l} =1$, holds.
\end{proposition}

\begin{proof}
Using the mean value theorem for every $\bx_1,\bx_2 \in \mathcal{X}$ there exist a $\bw =  \bx_1 + \beta \bx_2$, with $\beta \in (0,1)$ such that, 
$$|f(\bx_1) - f(\bx_2) | = | \nabla f(\bw) (\bx_1 -\bx_2)|.$$
By the Holder's inequality we have that
$$|f(\bx_1) - f(\bx_2) |  \leq \|{\nabla f(\bw)}\|_p \|\bx_1 -\bx_2 \|_q.$$
Since $\bw \in \mathcal{X}$ by definition, we have that
$$|f(\bx_1) - f(\bx_2) | \leq L_p \|\bx_1 -\bx_2 \|_q,$$
for $L_p = \max_{\bx \in\mathcal{X}}\|{\nabla f(\bx)}\|_p.$
\end{proof}

In order to test the empirical approximation of the Lipschitz constant detailed in Section 2.2 we use the Cosines function described in the experimental section of this work. The true $L_{\nabla}$ for this function  is  $8.808636$, that was calculated by maximizing  the norm of gradient of $f$ in a very fine grid.  We check the quality of our approximation to $L_{\nabla}$ for increasing sample size up to 50 observations, where the locations of the points are randomly selected along the domain of $f$ using a bivariate uniform distribution. The evaluations of $f$ at the selected locations were perturbed with Gaussian noise with standard deviations $\sigma=0,0.1,0.25$. In Figure \ref{figure:LCA} we show the results for 30 replicates of the experiment. The average approximation of L converges to the true $L_{\nabla}$, being this convergence slower when the evaluation errors increase.

\begin{figure}[h!]
	\centering
  \includegraphics[width=.45\textwidth]{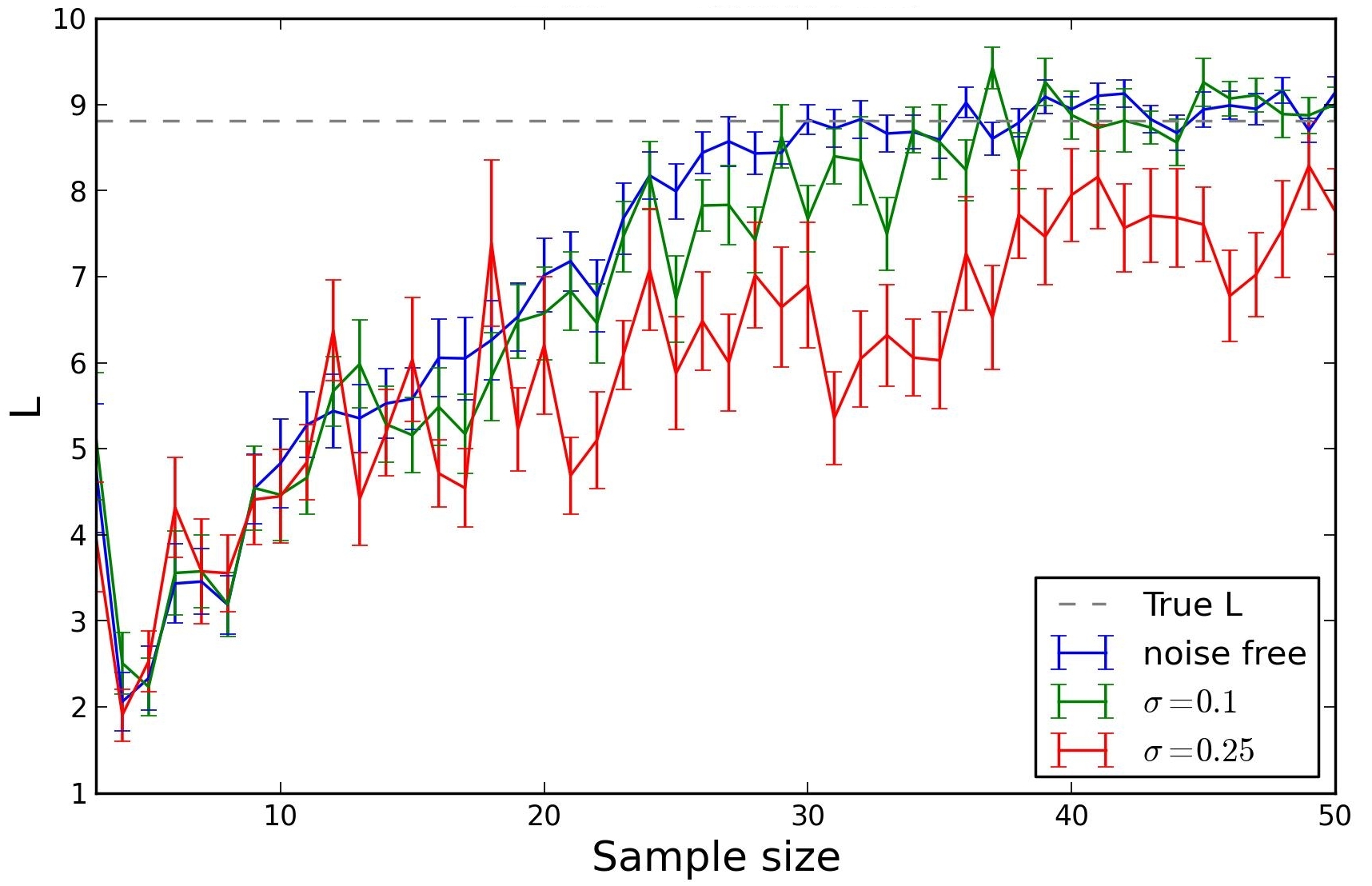}
\caption{Approximation of the Lipschitz constant in the cosines function using the GP-LCA method. We compare the convergence in the approximation for different noise levels  and increasing sample size. For each sample size show we show the average of 30 replications. Vertical bars represent the 95\% confidence interval for the average estimate. }
	\label{figure:LCA}
\end{figure}

\section{Detailed description of the experiments}\label{sec:functions}

\subsection{Synthetic functions}

Table~\ref{table:functions_test} contains the the details of the functions used in the experiments of this work. 

\begin{table}[t!]
\centering
\begin{tabular}{cccc}
 \hline 
Name & Function  &  $\mathcal{X}$  \\ \hline\hline
gSobol & $f(\bx) =\prod_{i=1}^d \frac{|4x_i-2|+a_1}{1+a_i}$  &  $[-5,5]^d$  \\
 Cosines &  $f(\bx) = 1- \sum_{i=1}^2 (g(x_i) - r(x_i) )$ &   $[0,5]^2$   \\
&{\small with $g(x_i) = (1.6x_i - 0.5)^2$,}  &  \\
& {\small  $r(x_i) = 0.3 \cos (3\pi (1.6 x_i -0.5))$.}  &  \\  \hline
\end{tabular}\caption{Functions used in the experimental section. All the parameters $a_i$ of the gSobol function are set to $a_i=1$ in the experiments.}\label{table:functions_test}
\end{table}

\subsection{Gene design experiment}

There is an increasing interest in the pharmacological industry in the the design of synthetic genes capable of transforming cells into `factories' able to produce drugs of interest. In this experiment we emulate a gene design process. 

The function to maximize is the production of cell proteins, that it is known depends on certain features of the gene sequences. We built a GP to link gene features and  protein production efficiency based the model described in \cite{gonzalez2014}.  A total of 71 gene features are considered, which correspond to the dimension of the final design space.  We validated the model with the remaining 2908 genes of the dataset and we used its posterior mean as the function to optimize.  We can understand this model as a mathematical surrogate of the cell behavior in which the mean evaluations play the role of physical wet-lab gene design tests, many of which can be run in parallel for the same price of one. 

\subsection{SVR parameter tuning experiment}

Support Vector Machines (SVR) for regression \cite{Drucker1997Support} with an EQ kernel, depend on three parameters: the kernel lengthscale ($\gamma$), the soft margin parameter ($C$) and the band size ($\epsilon$).  A proper choice of the parameters is crucial to guarantee a good performance of the SVR, which is typically done by minimizing the mean square error (RMSE) in a test dataset. This task can be expensive, specially for large datasets. We use BO to optimize the parameters of the SVR using the `Physiochemical' properties of protein tertiary structure' dataset available in the UCI Machine Learning repository \cite{Bache+Lichman:2013}.  This dataset has 45,730 instances and 9 continuous  attributes that are used to predict the coordinate root mean square distance (RMSD),  a measure that describes the distance per residue between to optimally aligned protein sequences. We trained the SVR using a randomly selected subset of 22,000 proteins and we tested the results of  using the rest. Every iteration takes around 300 seconds.

\end{document}